%% file: indmatcomp.tex
\documentclass[11pt]{article}

\usepackage{fullpage}
\usepackage{amsmath,amsfonts,amsthm,amssymb,xspace,bm, verbatim}
\usepackage{graphicx}
\usepackage{url,algorithm,algorithmic}

\newtheorem{lemma}{Lemma}
\newtheorem{theorem}[lemma]{Theorem}

\newtheorem{claim}[lemma]{Claim}
\newtheorem{defn}[lemma]{Definition}

\newcommand{\e}{\bm{e}}

\newcommand{\Tr}{\operatorname{Tr}}

\newcommand{\mvec}{\mathrm{vec}}
\newcommand{\note}[1]{\marginpar{\tiny *note in TeX*}}
\newcommand{\ignore}[1]{}

\newcommand{\calN}{{\cal N}}

\renewcommand{\phi}{\varphi}

\newcommand{\R}{\mathbb{R}}

\newcommand{\wt}[1]{\widetilde{#1}}
\newcommand{\tp}[1]{{#1}^{T}}

\newcommand{\wh}[1]{\widehat{#1}}

\newcommand{\dist}{\mathrm{dist}}

\newcommand{\twonorm}[1]{\left\| {#1} \right\|_2}

\DeclareMathOperator*{\argmin}{argmin}

\newcommand{\calA}{\mathcal{A}}

\newcommand{\E}{\mathbb{E}}

\renewcommand{\u}{\bm{u}}
\renewcommand{\v}{\bm{v}}

\newcommand{\Vh}{\widehat{V}}
\newcommand{\vhn}{\widehat{\v}_{h+1}}

\newcommand{\Ut}{U_{t}}

\newcommand{\Vo}{V_*}
\newcommand{\Uot}{\tp{U}_*}
\newcommand{\Vot}{\tp{V}_*}
\newcommand{\So}{\Sigma_*}
\newcommand{\Uo}{U_*}

\newcommand{\so}{\sigma_*}
\newcommand{\xti}{\tp{\x}_i}
\newcommand{\yti}{\tp{\y}_i}
\newcommand{\ytj}{\tp{\y}_j}

\newcommand{\vh}{\wh{\v}}
\newcommand{\ut}{\u_h}
\newcommand{\utt}{\tp{\u}_{h}}

\newcommand{\utnt}{\tp{\u}_{h+1}}

\newcommand{\vtnt}{\tp{\v}_{h+1}}

\newcommand{\hvtn}{\vh_{h+1}}
\newcommand{\hvtnt}{\tp{\vh}_{h+1}}
\newcommand{\Bo}{\widetilde{B}}

\newcommand{\uo}{\u_* }
\newcommand{\uot}{\tp{\u}_* }
\newcommand{\vo}{\v_* }
\newcommand{\vot}{\tp{\v}_* }

\newcommand{\ip}[2]{\langle #1, #2 \rangle}

\newcommand{\Wo}{W_*}
\newcommand{\Wot}{\tp{W}_*}

\newcommand{\z}{\bm{z}}
\newcommand{\x}{\bm{x}}
\newcommand{\y}{\bm{y}}
\renewcommand{\b}{\bm{b}}

\newcommand{\xit}{\tp{\x}_i}
\newcommand{\eit}{\tp{\e}_i}
\newcommand{\etj}{\tp{\e}_j}
\newcommand{\ejt}{\tp{\e}_j}
\newcommand{\ej}{\e_j}
\newcommand{\yi}{\y_i}
\newcommand{\yj}{\y_j}
\newcommand{\yit}{\tp{\y}_i}
\newcommand{\Vop}{\Vo^{\perp}}

\title{Provable Inductive Matrix Completion}

\author{Prateek Jain\\
{Microsoft Research India, Bangalore}\\
{prajain@microsoft.com}
\and
Inderjit S. Dhillon\\
{The University of Texas at Austin}\\
{inderjit@cs.utexas.edu}
}
\date{}
\begin{document}

\maketitle

\begin{abstract}
Consider a movie recommendation system where apart from the ratings information,  side information such as user's age or movie's genre is also available. Unlike  standard matrix completion, 
in this setting one should be able to predict inductively on new users/movies.  In this paper, we study the problem of inductive matrix completion in the exact recovery setting. That is, we assume that the ratings matrix is generated by applying feature vectors to a low-rank matrix and the goal is to recover back the underlying matrix. Furthermore, we generalize the problem to that of low-rank matrix estimation using rank-$1$ measurements. We study this generic problem and provide conditions that the set of measurements should satisfy so that the alternating minimization method (which otherwise is a non-convex method with no convergence guarantees) is able to recover back the {\em exact} underlying low-rank matrix. 

In addition to inductive matrix completion, we show that two other low-rank estimation problems can be studied in our framework: a) general low-rank matrix sensing using rank-$1$ measurements, and b) multi-label regression with missing labels. For both the problems, we provide novel and interesting bounds on the number of measurements required by alternating minimization to provably converges to the {\em exact} low-rank matrix. 
In particular, our analysis for the general low rank matrix sensing problem significantly improves the required storage and computational cost  than that required by the RIP-based matrix sensing methods \cite{RechtFP2007}. Finally, we provide empirical validation of our approach and demonstrate that  alternating minimization is able to recover the true matrix for the above mentioned problems using a small number of measurements.
\end{abstract}

\input{intro}
\input{sensing}

\input{comp}
\input{multi}
\input{exps}
\clearpage
\newpage
\begin{small}
\bibliographystyle{unsrt}
\bibliography{refs}
\end{small}
\clearpage
\newpage
\appendix
\input{app_sense}
\end{document}

%% file: intro.tex
\newcommand{\lrrom}{LRROM }
\section{Introduction}\label{sec:intro}
Motivated by the Netflix Challenge, recent research has addressed the problem of matrix completion where the goal is to recover the underlying low-rank ``ratings'' matrix by using a small number of observed entries of the matrix. However, the standard low-rank matrix completion formulation is applicable only to the transductive setting only, i.e., predictions are restricted to the existing users/movies only. However, several real-world recommendation systems have useful side-information available in the form of feature vectors for users as well as movies, and hence one should be able to make accurate predictions for new users and movies as well.

In this paper, we formulate and study the above mentioned problem which we call inductive matrix  completion, where other than a small number of observations from the ratings matrix, the feature vectors for users/movies are also available. We formulate the problem as that of recovering a low-rank  matrix $\Wo$ using observed entries $R_{ij}=\x_i^T \Wo \y_j$ and the user/movie feature vectors $\x_i$, $\y_j$. By factoring $\Wo=\Uo\Vot$, we see that this scheme constitutes a bi-linear prediction $(\x^T\Uo)(\Vot\y)$ for a new user/movie pair $(\x,\y)$. 

In fact, the above rank-$1$ measurement scheme also arises in several other important low-rank estimation problems such as: a) general low-rank matrix sensing in the signal acquisition domain, and b) multi-label regression problem with missing information. 

In this paper, we generalize the above three mentioned problems to the following low-rank matrix estimation problem that we call {\em Low-Rank matrix estimation using Rank One Measurements} ({\bf \lrrom}): recover the rank-$k$ matrix $\Wo \in \mathbb{R}^{d_1\times d_2}$ by using rank-$1$ measurements of the form: $$\b=[\x_1^T\Wo \y_1\ \ \x_2^T \Wo \y_2\ \ \dots\ \ \x_m^T \Wo\y_m]^T,$$
where $\x_i, \y_i$ are ``feature'' vectors and are provided along with the measurements $\b$.  

Now given measurements $\b$ and the feature vectors $\{\x_1\ \x_2\ \dots\ \x_m\}$, $Y=\{\y_1\ \y_2\ \dots\ \y_m\}$, a canonical way to recover $\Wo$ is to find a rank-$k$ matrix $W$ such that $\|\mathcal{A}(W)-\b\|_2$ is small. While the objective function of this problem is  simple least squares, the non-convex rank constraint makes it NP-hard, in general, to solve. In existing literature, there are two common approaches to handle such low-rank problems: a) Use trace-norm constraint as a proxy for the rank constraint and then solve the resulting non-smooth convex optimization problem, b) Parameterize $W$ as $W=UV^T$ and then alternatingly optimize for $U$ and $V$. 

The first approach has been shown to be successful for a variety of problems such as matrix completion \cite{CandesR2007, CandesT2009, Gross2009, KeshavanOM2009}, general low-rank matrix sensing \cite{RechtFP2007}, robust PCA \cite{CandesLMW11,ChandrasekaranSPW2009}, etc. However, the resulting convex optimization methods require computation of full SVD of matrices with potentially large rank  and hence do not scale to large scale problems. On the other hand, alternating minimization and its variants need to solve only least squares problems and hence are scalable in practice but might get stuck in a local minima. 
 However, \cite{JainNS13} recently showed that under standard set of assumptions, alternating minimization actually converges at a linear rate to the global optimum of two low-rank estimation problems: a) RIP measurements based general low-rank matrix sensing, and b) low-rank matrix completion. 

Motivated by its empirical as well as theoretical success, we  study a variant of alternating minimization  (with appropriate initialization) for the  above mentioned \lrrom problem. To analyze our general \lrrom problem, we present three key properties that a rank-$1$ measurement operator should satisfy. Assuming these properties, we show that the alternating minimization method converges to the global optima of \lrrom at a linear rate. We then study the  three problems individually and show that for each of the problems, the measurement operator indeed satisfies the conditions required by our general analysis and hence, for each of the problems alternating minimization converges to the global optimum at a linear rate. Below, we briefly describe the three application problems that we study and also our high-level result for each one of them: \\[8pt]
 {\bf (a) Efficient matrix sensing using Gaussian Measurements}: In this problem, $\x_i\in \mathbb{R}^{d_1}$ and $\y_i\in \mathbb{R}^{d_2}$ are sampled from a sub-Gaussian distribution and the goal is efficient acquisition and recovery of rank-$k$ matrix $\Wo$. Here, we show that if the number of measurements $m=\Omega(k^4 \beta^2 (d_1+d_2) \log(d_1+d_2))$, where $\beta=\so^1/\so^k$ is the condition number of $\Wo$. Then with high probability (w.h.p.), our alternating minimization based method will recover back $\Wo$ in linear time. 

Note that the problem of low-rank matrix sensing has been considered by several existing methods \cite{RechtFP2007, JainMD10, LeeBre09}, however most of these methods require the measurement operator to satisfy the Restricted Isometry Property (RIP) (see Definition~\ref{defn:rip}). Typically, RIP operators are constructed by sampling from distributions with bounded fourth moments and require $m=O(k(d_1+d_2)\log(d_1+d_2))$ measurements to satisfy RIP for a constant $\delta>0$. That is, the number of samples required to satisfy RIP are similar to the number of samples required by our method. 

Moreover, RIP based operators are typically dense, have a large memory footprint and make the algorithm computationally intensive. For example, assuming rank and $\beta$ to be constant, RIP based operators would require $O((d_1+d_2)d_1d_2))$ storage and computational time, as opposed to $O((d_1+d_2)^2)$ storage and computational time required by the rank-$1$ measurement operators. 
 However, a drawback of such rank-$1$ measurements is that, unlike RIP based operators, they are not universal, i.e., a new set of $\x_i, \y_i$ needs to be sampled for any given signal $\Wo$. \\[8pt]
{\bf (b) Inductive Matrix Completion}: As motivated earlier, consider a movie recommendation system with $n_1$ users and $n_2$ movies. Let $X\in \mathbb{R}^{n_1\times d_1},\ Y\in \mathbb{R}^{n_2\times d_2}$ be feature matrices of the users and the movies, respectively. Then, the user-movie rating $R_{ij}$ can be modeled as $R_{ij}=\tp{\x}_iW\y_j$ and the goal is to learn $W$ using a small number of random ratings indexed by the set of observations $\Omega \in [n_1]\times [n_2]$. Note that matrix completion is a special case of this problem when $\x_i=\e_i$ and $\y_j=\e_j$. Also, unlike standard matrix completion, accurate ratings can be predicted for users who have not rated any prior movies and vice versa. 

If the feature matrices $X, Y$ are incoherent and  the number of observed entries $|\Omega|=m\geq C\cdot (k^3 \beta^2 (d_1\cdot d_2)\log(d_1+d_2)$, then inductive matrix completion satisfies  the conditions required by our generic method and hence the global optimality result follows directly. Note that our analysis requires a quadratic number of samples, i.e., $\widetilde{O}(d_1\cdot d_2)$ samples (assuming $k$ to be a constant) for recovery. On the other hand, applying standard matrix completion would require $\widetilde{O}(n_1+n_2)$ samples. Hence, our analysis provides significant improvement if $d_1\cdot d_2\ll n_1+n_2$, i.e., when the number of features is significantly smaller than the total number of users and movies. \\[8pt]
{\bf (c) Multi-label Regression with Missing Data}: Consider a multi-variate regression problem, where the goal is to predict a set of (correlated) target variables $\bm{r}\in \R^{L}$ for a given $\x\in \R^{d_1}$.  We model this problem as a regression problem with low-rank parameters, i.e., $\bm{r}=\tp{W}\x$ where $W$ is a low-rank matrix. Given training data points $X=[\x_1\  \x_2\ \dots\ \x_{n_1}]$ and the associated target matrix $R$, $W$ can be learned using a simple least squares regression. However, in most real-world applications several of the entries in $R$ are missing and the goal is to be able to learn $W$ ``exactly''. 

Now, let the set of known entries $R_{ij}, (i,j)\in \Omega$ be sampled uniformly at random from $R$. Then we show that, by sampling $|\Omega|=m\geq k^3 \beta^2 \cdot (d_1\cdot L)\cdot \log(d_1+L)$ entries, alternating minimization recovers back $\Wo$ exactly. Note that a direct approach to this problem is to first recover the label matrix $R$ using standard matrix completion and then learn $\Wo$ from the completed label matrix. Such a method would require $\tilde{O}(n_1+L)$ samples of $R$. In contrast, our more unified approach requires $\tilde{O}(d_1\cdot L)$ samples. Hence, if the number of training points $n_1$ is much larger than the number of labels $L$, then our method  provides significant improvement over first completing the matrix and then learning the true low-rank matrix. 

We would like to stress that the above mentioned problems of inductive matrix completion and multi-label regression with missing labels have recently received a lot of attention from the machine learning community \cite{AbernethyBEV09, AgrawalGPV13}. However, to the best of our knowledge, our results are the first theoretically rigorous results that improve upon the sample complexity of first completing the target/ratings matrix and then learning the parameter matrix $\Wo$. 


{\bf Related Work}: Low-rank matrix estimation problems are pervasive and have innumerable real-life applications. Popular examples of low-rank matrix estimation problems include PCA, robust PCA, non-negative matrix approximation, low-rank matrix completion, low-rank matrix sensing etc. While in general low-rank matrix estimation that satisfies given (affine) observations is NP-hard, several recent results present  conditions under which the optimal solution can be recovered exactly or approximately \cite{CandesR2007, RechtFP2007, CandesT2009, AgarwalNW11, ChandrasekaranSPW2009, CandesLMW11, JainNS13, JainMD10}. 

Of these above mentioned low-rank matrix estimation problems, the most relevant problems to ours are those of matrix completion \cite{CandesR2007, KeshavanOM2009,JainNS13} and general matrix sensing \cite{RechtFP2007, JainMD10, LeeBre09}. The matrix completion problem is restricted to a given set of users and movies and hence does not generalize to new users/movies. On the other hand, matrix sensing methods require the measurement operator to satisfy the RIP condition, which at least for the current constructions, necessitate measurement matrices that have full rank, large number of random bits and hence high storage as well as computational time \cite{RechtFP2007}. Our work on general low-rank matrix estimation (problem (a) above) alleviates this issue as our measurements are only rank-$1$ and hence the low-rank signal $\Wo$ can be  encoded as well as decoded much more efficiently. Moreover, our result for inductive matrix completion generalizes the matrix completion work and provides,  to the best of our knowledge, the first theoretical results for the problem of inductive matrix completion. 


{\bf Paper Organization}: We formally introduce the problem of low-rank matrix estimation with rank-one measurements  in Section~\ref{sec:formulation}. We  provide our version of the alternating minimization method and then we present a {\em generic} analysis for  alternating minimization  when applied to such rank-one measurements based problems. Our results distill out certain key problem specific properties that would imply global optimality of alternating minimization. In the subsequent sections~\ref{sec:sense}, \ref{sec:imc}, and \ref{sec:mult}, we show that for each of our three problems (mentioned above) the required problem specific properties are satisfied and hence our alternating minimization method provides globally optimal solution. Finally, we provide empirical validation of our methods in Section~\ref{sec:exps}.

%% file: sensing.tex
\section{Low-rank Matrix Estimation using Rank-one Measurements}\label{sec:formulation}
\begin{algorithm}[t]
\caption{AltMin-\lrrom: Alternating Minimization for \lrrom}
\label{alg:altmin}
  \begin{algorithmic}[1]
    \STATE {\bf Input}: Measurements: $\b_{all}$, Measurement matrices: $\calA_{all}$, Number of iterations: $H$
    \STATE Divide $(\calA_{all}, \b_{all})$ into $2H+1$ sets (each of size $m$) with $h$-th set being $\calA^h=\{A_1^h,A_2^h,\dots,A_m^h\}$ and $\b^h=[b_1^h\ b_2^h\ \dots\ b_m^h]^T$
    \STATE {\bf Initialization}: $U_0=$top-$k$ left singular vectors of $\frac{1}{m}\sum_{i=1}^m b_i^0 A_i^0$
    \FOR{$h=0$ to $H-1$}
    \STATE $b\leftarrow b^{2h+1}, \calA\leftarrow \calA^{2h+1}$
    \STATE $\widehat{V}_{h+1}\leftarrow \argmin_{V\in \R^{d_2\times k}} \sum_i (b_i-\x_i^TU_h\tp{V}\y_i)^2$
    \STATE $V_{h+1}=QR(\widehat{V}_{h+1})$ //orthonormalization of $\widehat{V}_{h+1}$
    \STATE $b\leftarrow b^{2h+2}, \calA\leftarrow \calA^{2h+2}$
    \STATE $\widehat{U}_{h+1}\leftarrow \argmin_{U\in \mathbb{R}^{d_1\times k}} \  \sum_i (b_i-\x_i^TU\tp{V}_{h+1}\y_i)^2$
    \STATE $U_{h+1}=QR(\widehat{U}_{h+1})$ //orthonormalization of $\widehat{U}_{h+1}$

    \ENDFOR
    \STATE {\bf Output}: $W_H=U_H(\widehat{V}_H)^T$
  \end{algorithmic}
\end{algorithm}
Let $\calA: \mathbb{R}^{d_1\times d_2}\rightarrow \R^m$ be a linear measurement operator parameterized by $\calA=\{A_1, A_2, \dots, A_m\}$, where $A_i\in \mathbb{R}^{d_1\times d_2}$. Then, the linear measurements of a given matrix $W\in \R^{d_1\times d_2}$ are given by: 
\begin{equation}
  \label{eq:linA}
  \calA(W)=[\Tr(A_1^TW)\ \Tr(A_2^TW)\ \dots\ \Tr(A_m^TW)]^T,
\end{equation}
where $\Tr$ denotes the trace operator. 

In this paper, we mainly focus on the rank-$1$ measurement operators, i.e., $A_i=\x_i\tp{\y}_i, 1\leq i\leq m$ where $\x_i \in \R^{d_1}, \y\in \R^{d_2}$. Also, let $\Wo \in \R^{d_1\times d_2}$ be  a rank-$k$ matrix, with the singular value decomposition (SVD) $\Wo=\Uo\So\Vot$.

Then, given $\calA, \b$, the goal of the \lrrom problem is to recover back $\Wo$ efficiently. This problem can be reformulated as the following non-convex optimization problem:  
\begin{equation}
 \label{eq:r1ms}
 {\bf (\lrrom)}:\ \ \ \min_{W=U\tp{V}, U\in \R^{d_1\times k}, V\in \R^{d_2\times k}} \sum_{i=1}^m (b_i -\tp{\x}_i W \y_i)^2.
\end{equation}
Note that $W$ to be recovered is restricted to have at most rank-$k$ and hence $W$ can be re-written as $W=U\tp{V}$. 

We use the standard alternating minimization algorithm with appropriate initialization to solve the above problem \eqref{eq:r1ms} (see Algorithm~\ref{alg:altmin}). Note that the above problem is non-convex in $U, V$ and hence standard analysis would only ensure convergence to a local minima. However, \cite{JainNS13} recently showed that the alternating minimization method in fact converges to the global minima of two low-rank estimation problems: matrix sensing with RIP matrices and matrix completion. 

The rank-one operator given above does not satisfy RIP (see Definition~\ref{defn:rip}), even when the vectors $\x_i, \y_i$ are sampled from the normal distribution (see Claim~\ref{claim:notrip}). Furthermore, each measurement need not reveal exactly one entry of $\Wo$ as in the case of matrix completion. Hence, the proof of \cite{JainNS13} does not apply directly. However, inspired by the proof of \cite{JainNS13}, we distill out three key properties that the operator should satisfy, so that alternating minimization would  converge to the global optimum. 
\begin{theorem}\label{thm:general}
Let $\Wo =\Uo\So\Vot\in \mathbb{R}^{d_1\times d_2}$ be a rank-$k$ matrix with $k$-singular values $\so^1\geq \so^2\dots\geq \so^k$. 
Also, let $\mathcal{A}: \R^{d_1\times d_2}\rightarrow \R^m$ be a linear measurement operator parameterized by $m$ matrices, i.e., $\mathcal{A}=\{A_1, A_2, \dots, A_m\}$ where $A_i=\x_i\y_i^T$. Let $\mathcal{A}(W)$ be as given by \eqref{eq:linA}. 

Now, let $\mathcal{A}$ satisfy the following properties with parameter $\delta=\frac{1}{k^{3/2} \cdot \beta \cdot 100}\ \ $ ($\beta=\so^1/\so^k$): 
  \begin{enumerate}
  \item {\bf Initialization}: $\|\frac{1}{m}\sum_i b_i A_i-\Wo\|_2\leq \|\Wo\|_2\cdot \delta$. 
  \item {\bf Concentration of operators $B_x, B_y$}: Let $B_x=\frac{1}{m}\sum_{i=1}^m (\y_i^T\v)^2\x_i\tp{\x}_i$\\ and $B_y=\frac{1}{m}\sum_{i=1}^m (\tp{\x}_i\u)^2\y_i\tp{\y}_i$, where  $\u\in \mathbb{R}^{d_1}, \v\in \mathbb{R}^{d_2}$ are two unit vectors that are {\em independent} of randomness in $\x_i, \y_i,\ \forall i$. 
Then the following holds: $\|B_x-I\|_2\leq \delta$ and  $\|B_y-I\|_2\leq \delta$.
\item {\bf Concentration of operators $G_x, G_y$}: Let $G_x=\frac{1}{m}\sum_i (\y_i^T\v)(\y_i\v_\perp)\x_i\tp{\x}_i$,\\ $G_y=\frac{1}{m}\sum_i (\tp{\x}_i\u)(\tp{\u}_\perp\x_i)\y_i\tp{\y}_i$, where $\u,\u_\perp\in \mathbb{R}^{d_1},\ \  \v,\v_\perp\in \mathbb{R}^{d_2}$ are unit vectors, s.t., $\tp{\u}\u_\perp=0$ and $\tp{\v}\v_\perp=0$. Furthermore, let $\u, \u_\perp, \v, \v_\perp$ be independent of randomness in $\x_i, \y_i, \forall i$. 
Then, $\|G_x\|_2\leq \delta$ and $\|G_y\|_2\leq \delta$. 
  \end{enumerate}
Then, after $H$-iterations of the alternating minimization method (Algorithm~\ref{alg:altmin}), we obtain $W_H=U_H\tp{V}_H$ s.t., $\|W_H-\Wo\|_2\leq \epsilon$, where $H\leq 100\log(\|\Wo\|_F/\epsilon)$. 
\end{theorem}
\begin{proof}
We explain the key ideas of the proof by first presenting the proof for the special case of rank-$1$ $\Wo=\so \uo \vot$. Later in Appendix~\ref{app:general}, we extend the proof to general rank-$k$ case. 

Similar to \cite{JainNS13}, we first characterize the update for  $h+1$-th step iterates $\widehat{\v}_{h+1}$ of Algorithm~\ref{alg:altmin} and its normalized form $\v_{h+1}=\widehat{\v}_{h+1}/\|\widehat{\v}_{h+1}\|_2$. 

Now, by gradient of \eqref{eq:r1ms} w.r.t. $\vh$ to be zero while keeping $\ut$ to be fixed. That is, 
\begin{align}
 & \sum_{i=1}^m (b_i-\xti\ut\hvtnt\y_i)(\xti\ut)\y_i=0,\nonumber\\
i.e.,\ &\sum_{i=1}^m (\ut^T\x_i)\y_i(\so\yti\vo\uot\x_i-\yti\vh_{h+1}\ut^T\x_i)=0,\nonumber\\
i.e.,\ &\left(\sum_{i=1}^m (\xti\ut\utt\x_i)\y_i\yti\right)\hvtn=\sigma_*\left(\sum_{i=1}^m (\xti\ut\uot\x_i)\y_i\yti\right)\vo,\nonumber\\
i.e.,\ &\ \hvtn=\so (\uot\ut)\vo-\so B^{-1}( (\uot\ut) B-\Bo)\vo,
\label{eq:r1pm}
\end{align}
where, $$B=\frac{1}{m}\sum_{i=1}^m (\xti\ut\utt\x_i)\y_i\yti,\ \ \  \Bo=\frac{1}{m}\sum_{i=1}^m (\xti\ut\uot\x_i)\y_i\yti.$$ 

Note that \eqref{eq:r1pm} shows that $\hvtn$ is a perturbation of $\vo$ and the goal now is to bound the spectral norm of the perturbation term: 
\begin{equation}\label{eq:gub}\|G\|_2= \|B^{-1}( \uot\ut B-\Bo)\vo\|_2\leq \|B^{-1}\|_2\|\uot\ut B-\Bo\|_2\|\vo\|_2.\end{equation}
Now,, using Property 2 mentioned in the theorem, we get: 
\begin{equation}\label{eq:blb}\|B-I\|_2\leq 1/100, \ \ \ i.e.,\ \sigma_{min}(B)\geq 1-1/100,\ \ \ i.e.,\ \|B^{-1}\|_2\leq 1/(1-1/100).\end{equation}
Now, 
\begin{align}
(\uot\ut) B-\Bo&=\frac{1}{m}\sum_{i=1}^m \yi\yti \xti((\uot\ut)\ut\utt-\uo\utt)\x_i,\nonumber\\
&=\frac{1}{m}\sum_{i=1}^m \yi\yti \xti(\ut\utt-I)\uo\utt\x_i,\nonumber\\
&\stackrel{\zeta_1}{\leq} \frac{1}{100}\|(\ut\utt-I)\uo\|_2 \|\utt\|_2= \frac{1}{100}\sqrt{1-(\utt\uo)^2},\label{eq:cup}
\end{align}
where $\zeta_1$ follows by observing that $(\ut\utt-I)\uo$ and $\ut$ are orthogonal set of vectors and then using Property 3 given in the Theorem~\ref{thm:general}. Hence, using \eqref{eq:blb}, \eqref{eq:cup}, and  $\|\vo\|_2=1$ along with \eqref{eq:gub}, we get: 
\begin{equation}
  \label{eq:gub1}
  \|G\|_2\leq \frac{1}{99}\sqrt{1-(\utt\uo)^2}. 
\end{equation}

We are now ready to lower bound the component of $\widehat{\v}_h$ along the correct direction $\vo$ and the component of $\widehat{\v}_h$ that is perpendicular to the optimal direction $\vo$. 

Now, by left-multiplying \eqref{eq:r1pm}  by $\vo$ and using \eqref{eq:blb} we obtain: 
\begin{equation}
  \label{eq:vvo}
  \vot\vhn=\so(\utt\uo)-\so\vot G\geq \so(\utt\uo)-\frac{\so}{99}\sqrt{1-(\utt\uo)^2}. 
\end{equation}
Similarly, by multiplying \eqref{eq:r1pm} by $\vo^\perp$, where $\vo^\perp$ is a unit norm vector that is orthogonal to $\vo$, we get: 
\begin{equation}
  \label{eq:vp}
  \ip{\vo^\perp}{\vhn}\leq \frac{\so}{99}\sqrt{1-(\utt\uo)^2}. 
\end{equation}
Using \eqref{eq:vvo}, \eqref{eq:vp}, and $\|\vhn\|_2^2=(\vot\vhn)^2+((\vo^\perp)^T\vhn)^2$, we get: 
\begin{align}
  1-(\vtnt\vo)^2 &= \frac{\ip{\vo^\perp}{\vhn}^2}{\ip{\vo}{\vhn}^2+\ip{\vo^\perp}{\vhn}^2},\nonumber\\
&\leq \frac{1}{99\cdot 99\cdot (\utt\uo-\frac{1}{99}\sqrt{1-(\utt\uo)^2})^2+1} (1-(\ut\uo)^2).
  \label{eq:vnorm}
\end{align}
Also, using Property 1 of Theorem~\ref{thm:general}, for $S=\frac{1}{m}\sum_{i=1}^m b_i A_i$, we get: $\|S\|_2\geq \frac{99\so}{100}$. Moreover, by multiplying $S-\Wo$ by $\u_0$ on left and $\v_0$ on the right and using the fact that $(\u_0, \v_0)$ are the largest singular vectors of $S$, we get: $\|S\|_2-\so\v_0^T\vo \u_0^T\uo\leq \so/100$. Hence, $\u_0^T\uo\geq 9/10$. 

Using the \eqref{eq:vnorm} along with the above given observation and by the ``inductive'' assumption $\u_h^T\uo\geq \u_0^T\uo\geq 9/10$ (proof of the inductive step follows directly from the below equation) , we get: 
\begin{equation}
  \label{eq:dist}
  1-(\vtnt\vo)^2 \leq \frac{1}{2}  (1-(\utt\uo)^2). 
\end{equation}
Similarly, we can show that $ 1-(\utnt\uo)^2 \leq \frac{1}{2}  (1-(\vtnt\vo)^2).$ Hence, after $H=O(\log(\so/\epsilon))$ iterations, we obtain $W_H=\u_H \widehat{\v}^T_H$, s.t., $\|W_H-\Wo\|_2\leq \epsilon$. 
\end{proof}
Note that we require intermediate vectors $\u, \v, \u_\perp, \v_\perp$ to be independent of randomness in $A_i$'s. Hence, we partition $\calA_{all}$ into $2H+1$ partitions and at each step $(\calA^h, \b^h)$ and $(\calA^{h+1}, \b^{h+1})$ are supplied to the algorithm. This implies that the measurement complexity of the algorithm is given $m\cdot H=m\log(\|\Wo\|_F/\epsilon)$. That is, given $O(m \log(\|(d_1+d_2)\Wo\|_F)$ samples, we can estimate matrix $W_H$, s.t., $\|W_H-\Wo\|_2 \leq \frac{1}{(d_1+d_2)^c}$, where $c>0$ is any  constant. 

\section{Rank-one Matrix Sensing using Gaussian Measurements}\label{sec:sense}
In this section, we study the problem of sensing general low-rank  matrices which is an important problem in the domain of signal acquisition \cite{RechtFP2007} and has several applications in a variety of areas like control theory, computer vision, etc.  For this problem, the goal is to {\em design} the measurement matrix $A_i$ as well as recovery algorithm, so that the true low-rank signal $\Wo$ can be  recovered back from the given linear measurements. 

Consider a measurement operator $\calA_{Gauss}=\{A_1, A_2, \dots, A_m\}$ where each measurement matrix $A_i=\x_i\y_i^T$ is sampled using normal distribution, i.e., $\x_i \sim N(0, I),\ y_i\sim N(0, I), \forall i$. Now, for this operator $\calA_{Gauss}$, we show that if $m=\Omega(k^4 \beta^2 \cdot (d_1+d_2) \cdot \log^2(d_1+d_2))$, then w.p. $\geq 1-1/(d_1+d_2)^{100}$,  {\em any} fixed rank-$k$ matrix $\Wo$ can be recovered by  AltMin-\lrrom (Algorithm~\ref{alg:altmin}). Here $\beta=\so^1/\so^k$ is the condition number of $\Wo$. That is, using  nearly linear number of measurements in $d_1, d_2$, one can exactly recover the $d_1\times d_2$ rank-$k$ matrix $\Wo$. 


Note that 
several similar recovery results for the matrix sensing problem already exist in the literature that guarantee exact recovery using $\Omega(k (d_1+d_2)\log(d_1+d_2))$ measurements \cite{RechtFP2007,LeeBre09,JainMD10}. However, we would like to  stress that all the above mentioned existing results assume that the measurement operator $\calA$ satisfies the Restricted Isometry Property (RIP) defined below: 
  \begin{defn}
    \label{defn:rip}
A linear operator $\calA: \mathbb{R}^{d_1\times d_2}\rightarrow \mathbb{R}^{m}$ satisfies RIP iff, $\forall W$ s.t. $rank(W)\leq k$, the following holds: 
$$(1-\delta_k)\|W\|_F^2\leq \|\calA(W)\|_F^2 \leq (1+\delta_k) \|W\|_F^2,$$
where $\delta_k>0$ is a constant dependent only on $k$. 
  \end{defn}
Most current constructions of RIP matrices require each $A_i$ to be sampled from a zero mean distribution with bounded fourth norm which implies that they have almost {\em full} rank. That is, such operators require $O(m d_1 d_2)$ memory just to store the operator, i.e., the storage requirement is cubic in $d_1+d_2$. Consequently signal acquisition as well as recovery time for these algorithms is also at least cubic in $d_1+d_2$. In contrast, our proposed rank-$1$ measurements require only $O(m (d_1+d_2))$ storage and computational time. Hence, the proposed method makes the signal acquisition as well as signal recovery at least an order of magnitude faster . 

Naturally, this begs the question whether we can show that our rank-$1$ measurement operator $\calA_{Gauss}$ satisfies RIP, so that the existing analysis for RIP based low-rank matrix sensing can be used \cite{JainNS13}. We answer this question in the negative, i.e., for $m=O((d_1+d_2)\log(d_1+d_2))$, $\calA_{Gauss}$ does not satisfy RIP even for rank-$1$ matrices (with high probability): 

\begin{claim}\label{claim:notrip}
  Let $\calA_{Gauss}=\{A_1, A_2, \dots A_m\}$ be a measurement operator with each $A_i=\x_i\y_i^T$, where 
$\x_i \in \R^{d_1}\sim \calN(0, I),\ \y_i \in \R^{d_2}\sim \calN(0, I), 1\leq i\leq m$. Let $m=O((d_1+d_2)\log^c(d_1+d_2)$, for any constant $c>0$. Then, with probability at least $1-1/m^{10}$, $\calA_{Gauss}$ does not satisfy RIP for rank-$1$ matrices with a constant $\delta$. 
\end{claim}
\begin{proof}[Proof of Claim~\ref{claim:notrip}]
The main idea behind our proof is to show that there exists two rank-$1$ matrices $Z_U, Z_L$ s.t. $\|\mathcal{A}_{Gauss}(Z_U)\|_2^2$ is large  while $\|\mathcal{A}_{Gauss}(Z_L)\|_2^2$ is much smaller than $\|\mathcal{A}_{Gauss}(Z_U)\|_2^2$. 

In particular, let $Z_U=\x_1 \y_1^T$ and let $Z_L=\u \v^T$ where $\u, \v$ are sampled from normal distribution independent of $X, Y$. 
Now, $$\|\mathcal{A}_{Gauss}(Z_U)\|_2^2=\sum_{i=1}^m \|\x_1\|_2^4\|\y_1\|_2^4+\sum_{i=2}^m (\x_1^T\x_i)^2(\y_1^T\y_i)^2.$$
Now, as  $\x_i, \y_i, \forall i$ are multi-variate normal random variables, $\|\x_1\|_2^4\|\y_1\|_2^4\geq 0.5 d_1^2d_2^2$ w.p. $\geq 1-2\exp(-d_1-d_2)$. 
\begin{equation}\|\mathcal{A}_{Gauss}(Z_U)\|_2^2\geq .5 d_1^2 d_2^2.\label{eq:nrlb}\end{equation}
Moreover, $\|Z_U\|_F^2\leq 2d_1d_2$ w.p. $\geq 1-2\exp(-d_1-d_2)$.

Now, consider $$\|\mathcal{A}_{Gauss}(Z_L)\|_2^2=\sum_{i=2}^m (\u^T\x_i)^2(\v^T\y_i)^2,$$
where $Z_L=\u\v^T$ and $\u, \v$ are sampled from standard normal distribution, independent of $\x_i, \y_i, \forall i$. Since, $\u, \v$ are independent of $\u^T\x_i \sim N(0, \|\u\|_2)$ and $\v^T\y_i\sim N(0, \|\v\|_2)$. Hence, w.p. $\geq 1-1/m^{3}$, $|\u^T\x_i| \leq \log(m)\|\u\|_2, |\v^T\y_i| \leq \log(m)\|\v\|_2, \forall i\geq 2$.  Moreover, w.p. $\geq 1-\exp(-d_1-d_2)$, $\|\u\|_2\leq 2 \sqrt{d_1}$ and $\|\v\|_2\leq 2 \sqrt{d_2}$. 
That is, w.p. $1-1/m^3$: \begin{equation}\label{eq:nrub}\|\mathcal{A}_{Gauss}(Z_L)\|_2^2\leq 4m\cdot d_1\cdot d_2 \log^4m .\end{equation}
Furthermore, $\|Z_L\|_F^2\leq 2d_1d_2$ w.p.  $\geq 1-2\exp(-d_1-d_2)$. 

Using \eqref{eq:nrlb}, \eqref{eq:nrub}, we get that w.p.  $\geq 1-2/m^3-10 \exp(-d_1-d_2)$: $$40m \log^4m\leq  \|\calA_{Gauss}(Z/\|Z\|_F)\|^2\leq .05 d_1d_2.$$
Now, for RIP to be satisfied with a constant $\delta$, the lower and upper bound on $\|\calA_{Gauss}(Z/\|Z\|_F)\|^2$ for all rank-$1$ $Z$ should be at most a constant factor apart. However, the above equation clearly shows that the upper and lower bound can match only when $m=\Omega(d_1d_2/\log(5d_1d_2))$. Hence, for $m$ that is at most linear in both $d_1$, $d_2$, RIP cannot be satisfied with probability $\geq 1-1/(d_1+d_2)^3$. 
\end{proof}

Now, even though $\calA_{Gauss}$ does not satisfy RIP, we can still show that $\calA_{Gauss}$ satisfies the three properties mentioned in the Theorem~\ref{thm:general}.  and hence we can use Theorem~\ref{thm:general} to obtain the exact recovery result. 
\begin{lemma}[Rank-One Gaussian Measurements]\label{lemma:sense}
Let $\calA_{Gauss}=\{A_1, A_2, \dots A_m\}$ be a measurement operator with each $A_i=\x_i\y_i^T$, where 
$\x_i \in \R^{d_1}\sim \calN(0, I),\ \y_i \in \R^{d_2}\sim \calN(0, I), 1\leq i\leq m$. Let $m=\Omega(k^4 \beta^2 (d_1+d_2)\log^3(d_1+d_2)$. Then, Property 1, 2, 3 required by Theorem~\ref{thm:general} are satisfied with probability at least $1-1/(d_1+d_2)^{100}$. 
\end{lemma}

\begin{proof}[Proof of Lemma~\ref{lemma:sense}]
  We divide the proof into three parts where each part proves a property mentioned in Theorem~\ref{thm:general}. 
\begin{proof}[Proof of Property 1]
Now,
$$S=\frac{1}{m}\sum_{i=1}^m b_i \x_i \tp{\y}_i=\frac{1}{m}\sum_{i=1}^m \x_i\tp{\x}_i\Uo\So\tp{V}_*\y_i\tp{\y}_i=\frac{1}{m}\sum_{i=1}^m Z_i,$$
where $Z_i=\x_i\tp{\x}_i\Uo\So\Vot\y_i\y_i^T$. Note that $\mathbb{E}[Z_i]=\Uo\So\Vot$. Also, both $\x_i$ and $\y_i$ are spherical Gaussian variables and hence are rotationally invariant. Therefore, wlog, we can assume that $\Uo=[\e_1 \e_2 \dots \e_k]$ and $\Vo=[\e_1 \e_2 \dots \e_k]$ where $e_i$ is the $i$-th canonical basis vector. 

As $S$ is a sum of $m$ random matrices, the goal is to apply matrix concentration bounds to show that $S$ is close to $\E[S]=W=\Uo\So\Vot$ for large enough $m$. To this end, we use Theorem~\ref{thm:tropp} by \cite{Tropp11} given below. However, Theorem~\ref{thm:tropp} requires bounded random variable while $Z_i$ is an unbounded variable. We handle this issue by clipping $Z_i$ to ensure that its spectral norm is always bounded. In particular, consider the following random variable: 
\begin{equation}
  \label{eq:clipx}
  \wt{x}_{ij}= \begin{cases}x_{ij},& \ |x_{ij}|\leq C \sqrt{\log(m(d_1+d_2))},\\
0,&\text{ otherwise},\end{cases}\end{equation}
where $x_{ij}$ is the $j$-th co-ordinate of $\x_i$. Similarly, define: 
\begin{equation}\label{eq:clipy}
  \wt{y}_{ij}= \begin{cases}y_{ij},& \ |y_{ij}|\leq C \sqrt{\log(m(d_1+d_2))},\\
0,&\text{ otherwise}.\end{cases}
\end{equation}
Note that, $\mathbb{P}(x_{ij}=\wt{x}_{ij})\geq 1-\frac{1}{(m(d_1+d_2))^C}$ and $\mathbb{P}(y_{ij}=\wt{y}_{ij})\geq 1-\frac{1}{(m(d_1+d_2))^C}$. Also, $\wt{x}_{ij}, \wt{y}_{ij}$ are still symmetric and independent random variables, i.e., $\E[\wt{x}_{ij}]=\E[\wt{y}_{ij}]=0,\ \forall i, j$. Hence, $\E[\wt{x}_{ij}\wt{x}_{i\ell}]=0, \forall j\neq \ell$. Furthermore, $\forall j$, 
\begin{align}
  \E[\wt{x}_{ij}^2]&=\E[x_{ij}^2]-\frac{2}{\sqrt{2\pi}}\int_{C\sqrt{\log(m(d_1+d_2))}}^{\infty} x^2 \exp(-x^2/2)dx, \nonumber\\
&=1-\frac{2}{\sqrt{2\pi}}\frac{C\sqrt{\log(m(d_1+d_2))}}{(m(d_1+d_2))^{C^2/2}}-\frac{2}{\sqrt{2\pi}}\int_{C\sqrt{\log(m(d_1+d_2))}}^{\infty} \exp(-x^2/2)dx, \nonumber\\
&\geq 1-\frac{2C\sqrt{\log(m(d_1+d_2))}}{(m(d_1+d_2))^{C^2/2}}. \label{eq:scx}
\end{align}
Similarly, 
\begin{equation}
  \label{eq:scy}
  \E[\wt{y}_{ij}^2]\geq 1-\frac{2C\sqrt{\log(m(d_1+d_2))}}{(m(d_1+d_2))^{C^2/2}}. 
\end{equation}

Now, consider RV, $\wt{Z}_i=\wt{\x}_i\tp{\wt{\x}}_i \Uo\So\Vot \wt{\y}_i\tp{\wt{\y}}_i.$ Note that, $\|\wt{Z}_i\|_2\leq C^4 \sqrt{d_1d_2}k\log^2(m(d_1+d_2)) \so^1$ and $\|\E[\wt{Z}_i]\|_2\leq \so^1$. Also, 
\begin{align}
  \|\E[\wt{Z}_i\tp{\wt{Z}}_i]\|_2&=\|\E[\|\wt{\y}_i\|_2^2\wt{\x}_i\tp{\wt{\x}}_i \Uo\So\Vot \wt{\y}_i\tp{\wt{\y}}_i\Vo\So\tp{U}_*\wt{\x}_i\tp{\wt{\x}_i}]\|_2,\nonumber\\
&\leq C^2d_2\log(m(d_1+d_2))\E[\wt{\x}_i\tp{\wt{\x}}_i \Uo\So^2\Uot\wt{\x}_i\tp{\wt{\x}}_i]\|_2,\nonumber\\
&\leq C^2d_2\log(m(d_1+d_2))(\so^1)^2\|\E[\|\Uot\wt{\x}_i\|_2^2 \wt{\x}_i\tp{\wt{\x}}_i]\|_2,\nonumber\\
&\leq C^4kd_2\log^2(m(d_1+d_2))(\so^1)^2. 
\end{align}
Similarly, 
\begin{equation}
  \label{eq:expz2}
  \|\E[\wt{Z}_i]\E[\tp{\wt{Z}}_i]\|_2\leq (\sigma_*^{max})^2. 
\end{equation}
Similarly, we can obtain bounds for $\|\E[\tp{\wt{Z}}_i\wt{Z}_i]\|_2$, $\|\tp{\E[\wt{Z}_i]}\E[\wt{Z}_i]\|_2$. 

Finally, by selecting $m=\frac{C_1 k (d_1+d_2) \log^2(d_1+d_2)}{\delta^2}$ and applying Theorem~\ref{thm:tropp} we get (w.p. $1-\frac{1}{(d_1+d_2)^{10}}$), 
\begin{equation}
  \label{eq:tz2}
  \|\frac{1}{m}\sum_{i=1}^m \wt{Z}_i - \E[\wt{Z}_i]\|_2 \leq \delta. 
\end{equation}
Note that $\E[\wt{Z}_i]=\E[\wt{x}_{i1}^2]\E[\wt{y}_{i1}^2]\Uo\So\Vot $. Hence, by using \eqref{eq:tz2}, \eqref{eq:scx}, \eqref{eq:scy}, 
 $$\|\frac{1}{m}\sum_{i=1}^m \wt{Z}_i - \Uo\So\Vot\|_2\leq \delta+\frac{\sigma_*^{1}}{(d_1+d_2)^{100}}.$$
Finally, by  observing that by selecting $C$ to be large enough in the definition of $\wt{\x_i},\wt{\y}_i$ (see \eqref{eq:clipx}, \eqref{eq:clipy}), we get $P(\|Z_i-\wt{Z_i}\|_2=0)\geq 1-\frac{1}{(d_1+d_2)^5}$. Hence, by assuming $\delta$ to be a constant wrt $d_1, d_2$ and by union bound, w.p. $1-\frac{2\delta^{10}}{(d_1+d_2)^5}$, 
$$\|\frac{1}{m}\sum_{i=1}^m Z_i -\Wo\|_2\leq 5\delta\|\Wo\|_2.$$
Now, the theorem follows directly by setting $\delta= \frac{1}{100 k^{3/2} \beta}$. 
\end{proof}
\end{proof}

Global optimality of the rate of convergence of the Alternating Minimization procedure for this problem now follows directly by using Theorem~\ref{thm:general} with the above given lemma. We would like to note that while the above result shows that the $\calA_{Gauss}$ operator is almost as powerful as the RIP based operators for matrix sensing, there is one critical drawback: while RIP based operators are universal that is they can be used to recover any rank-$k$ $\Wo$, $\calA_{Gauss}$ needs to be resampled for each $\Wo$. We believe that the two operators are at two extreme ends of randomness vs universality trade-off and intermediate operators with higher success probability but using larger number of random bits should be possible.

%% file: comp.tex
\section{Inductive Matrix Completion}\label{sec:imc}
In this section, we study the problem of inductive matrix completion which is another important application of the \lrrom problem. Consider a movie recommender system which contains $n_1$ users and $n_2$ movies and let $R\in \mathbb{R}^{n_1\times n_2}$ be the corresponding ``true'' ratings matrix. 
The standard matrix completion methods  only utilize the samples from the ratings matrix $R$ and ignore the side-information that might be present in the system such as, demographic information of the user or genre of the movie. This  restricts the usage of matrix completion to the transductive setting only. 

Recently, \cite{AbernethyBEV09} studied a generalization of the low-rank matrix completion problem where $R_{ij}$ is modeled as $R_{ij}=\x_i^T \Wo \y_j$; where $\x_i, \y_j$ are the feature vectors of users and movies, respectively. Using benchmark datasets, they showed empirically that their method  outperforms  traditional matrix completion methods. However, to the best of our knowledge, there is no existing theoretical analysis of such an inductive approach. 

Now, since $R$ is a rank-$k$ matrix, one can still apply standard matrix completion results to recover $R$ and hence $\Wo$. Assuming that the observed index set $\Omega$ is sampled uniformly from $[n_1]\times [n_2]$ and that $R$ is incoherent, a direct application of the matrix completion methods would require $|\Omega|\geq C(k(n_1+n_2)\log(n_1+n_2))$ samples to be known. Now, if $d_1+d_2 \ll n_1+n_2$ then this means that many more samples are required than the total degrees of freedom in $\Wo$ which is $O(k(d_1+d_2))$. 

Hence, a natural question here is can the above given sample complexity bound  be improved? 
Below, we provide the answer to this question in affirmative. In particular, we show that by using the feature vectors AltMin-\lrrom (see Algorithm~\ref{alg:altmin}) can recover the true matrix $\Wo$ using  $O(kd_1d_2\log(d_1d_2))$ random samples. Now, if $d_1d_2\ll n_1+n_2$, then our method requires significantly lesser number of samples than the standard matrix completion methods. Furthermore, this implies that several users/movies need not have even {\em one} known rating, i.e, the method can be applied to the inductive setting as well. We note that our sample size requirement is still larger than the information theoretically optimal requirement which is $O(k(d_1+d_2)\log(d_1+d_2))$. We leave further reduction in the sample complexity as an open problem. 

Similar to the previous section, we utilize our general theorem for optimality of the \lrrom problem to provide a convergence analysis of the inductive matrix completion method. In particular, we provide the following lemma which shows that assuming $X, Y$ to be incoherent (see Definition~\ref{defn:incoherence}), the above mentioned inductive matrix completion operator also satisfies Properties 1, 2, 3 required by Theorem~\ref{thm:general}. Hence, AltMin-\lrrom (Algorithm~\ref{alg:altmin}) converges to the global optimum in $O(\log (\|\Wo\|_F/\epsilon))$ iterations. We first provide the definition of incoherent matrices. 
\begin{defn}\label{defn:incoherence}
$X\in \mathbb{R}^{d\times n}$ ($d<n$) is $\mu$-incoherent if: $\|U_X^i\|_2\leq \frac{\mu\sqrt{d}}{\sqrt{n}}, 1\leq i\leq d$, where $X^T=U_X\Sigma_X\tp{V}_X$ is the SVD of $X^T$ and $U_X^i\in \R^d$ is the $i$-th row of $U_X\in \mathbb{R}^{n\times d}$. 
\end{defn}
\begin{lemma}\label{lemma:indmc}
Let both $X\in \mathbb{R}^{d_1\times n_1}$ and $Y\in \mathbb{R}^{d_2\times n_2}$ be $\mu$-incoherent  matrices. Let $R=\tp{X}\Wo Y$ be the ``ratings'' matrix and let $\Wo\in \mathbb{R}^{d_1\times d_2}$ be any fixed rank-$k$ matrix. Let $\Omega$ be a uniformly random subset of $[n_1]\times [n_2]$, s.t., $|\Omega|=m\geq C  k^3\cdot \beta^2\cdot d_1 d_2\cdot \log(d_1+d_2)$, where $\beta=\sigma_R^1/\sigma_R^k$ is the condition number of $R$. Then, w.p. $\geq 1-1/(d_1+d_2)^{100}$, the measurement operators $A_{ij}=\sqrt{n_1n_2}\x_i \tp{\y}_j$ satisfy\footnote{We multiply $\x_i, \y_j$ by $\sqrt{n_1}, \sqrt{n_2}$ for normalization so that $\E_{i}[n_1\x_i\x_i^T]=I$ and $\E_{j}[n_2\yj\ytj]=I$.} Properties 1,2,3 required by Theorem~\ref{thm:general}. 
\end{lemma}
\begin{proof}

We first observe that both $X, Y$ can be thought of as orthonormal matrices. The reason being, $X^T\Wo Y =U_X \Sigma_X V_X^T \Wo V_Y \Sigma_Y U_Y^T$, where $X^T=U_X \Sigma_X V_X^T$ and $Y^T=U_Y \Sigma_Y V_Y^T$. Hence, $R=X^T \Wo Y=U_X (\Sigma_X V_X^T \Wo V_Y \Sigma_Y) U_Y^T$. That is, $U_X $, $U_Y$ can be treated as the true ``X'', ``Y'' matrices and $\Wo\leftarrow (\Sigma_X V_X^T \Wo V_Y \Sigma_Y)$ can be thought of as $\Wo$.  Then the ``true'' $\Wo$ can be recovered using the obtained $W_H$ as: $W_H\leftarrow V_X \Sigma_X^{-1} W_H \Sigma_Y^{-1} V_Y^T$. We also note that such a transformation implies that the condition number of $R$ and that of $\Wo\leftarrow (\Sigma_X V_X^T \Wo V_Y \Sigma_Y)$ are exactly the same. Hence, we prove the theorem with the assumption that $X$, $Y$ are orthonormal and that $\beta$ is the condition number of $\Wo$. 

We now present the proof for each of the three properties mentioned in Theorem~\ref{thm:general}. 
\begin{proof}[Proof of Property 1]
As mentioned above, wlog, we can assume that  both $X, Y$ are orthonormal matrices and  that the condition number of $R$ is same as condition number of $\Wo$. 

We first recall the definition of $S$:
$$S=\frac{n_1n_2}{m}\sum_{(i,j)\in \Omega}^m \x_i\xti \Uo\So\Vot \y_j\ytj=\frac{n_1n_2}{m}\sum_{(i,j)\in \Omega}^m Z_{ij},$$
where $Z_{ij}=\x_i\xti \Uo\So\Vot \y_j\ytj = X \e_i \eit X^T \Uo\So\Vot Y \e_j \ejt Y^T$, where $\e_i, \e_j$ denotes the $i$-th, $j$-th canonical basis vectors, respectively. 

Also, since $(i,j)$ is sampled uniformly at random from $[n_1]\times [n_2]$. Hence, $\E_{i}[\e_i\eit]=\frac{1}{n_1}I$ and $\E_{j}[\e_j\ejt]=\frac{1}{n_2}I$. That is, 
$$\E_{ij}[Z_{ij}]=\frac{1}{n_1n_2} XX^T\Uo\So\Vot YY^T=\Uo\So\Vot=\Wo/(n_1\cdot n_2),$$ where $XX^T=I$, $YY^T=I$ follows by orthonormality of both $X$ and $Y$. 

We now use the matrix concentration bound of Theorem~\ref{thm:tropp} to bound $\|S-\Wo\|_2$. To apply the bound of Theorem~\ref{thm:tropp}, we first need to bound the following two quantities: 
\begin{itemize}
\item {\bf Bound $\max_{ij} \|Z_{ij}\|_2$}: Now, $$\|Z_{ij}\|_2=\|\x_i\xti \Uo\So\Vot \y_j\ytj\|_2\leq \so^1 \|\x_i\|_2^2 \|\y_j\|_2^2\leq \frac{\so^1\mu^4d_1d_2}{n_1n_2},$$
where the last inequality follows using incoherence of $X, Y$. 
\item {\bf Bound $\|\sum_{(i,j)\in \Omega}E[Z_{ij}Z_{ij}^T]\|_2$ and  $\|\sum_{(i,j)\in \Omega}E[Z_{ij}^TZ_{ij}]\|_2$}: 

We first consider $\|\sum_{(i,j)\in \Omega}E[Z_{ij}Z_{ij}^T]\|_2$:
\begin{align}
&\left\|\sum_{(i,j)\in \Omega}E[Z_{ij}Z_{ij}^T]\right\|_2=\left\|\sum_{(i,j)\in \Omega}\E[\x_i\xti\Wo\yj\ytj\yj\ytj\Wo^T \x_i \xti]\right\|_2,\nonumber\\
&\stackrel{\zeta_1}{\leq} \frac{\mu^2 d_2}{n_2}\left\|\sum_{(i,j)\in \Omega}\E[\x_i\xti\Wo\yj\ytj\Wo^T \x_i \xti]\right\|_2\stackrel{\zeta_2}{=}\frac{\mu^2 d_2}{n_2^2}\left\|\sum_{(i,j)\in \Omega}\E[\x_i\xti\Wo\Wo^T \x_i \xti]\right\|_2,\nonumber\\
&\stackrel{\zeta_3}{\leq }\frac{(\so^1)^2\mu^4 d_1d_2}{n_1n_2^2}\left\|\sum_{(i,j)\in \Omega}\E[\x_i\xti]\right\|_2\stackrel{\zeta_4}{= }\frac{(\so^1)^2\mu^4 d_1d_2}{n_1^2n_2^2}\cdot m,\label{eq:imc_var}
\end{align}
where $\zeta_1$, $\zeta_3$ follows by using incoherent of $X, Y$ and $\|\Wo\|_2\leq \so^1$. $\zeta_2, \zeta_4$ follows by using $\E_{i}[\e_i\eit]=\frac{1}{n_1}I$ and $\E_{j}[\e_j\ejt]=\frac{1}{n_2}I$. 

Now, bound for $\|\sum_{(i,j)\in \Omega}E[Z_{ij}^TZ_{ij}]\|_2$ also turns out to be exactly the same and can be easily computed using exactly same arguments as above. 
\end{itemize}
Now, by applying Theorem~\ref{thm:tropp} and using the above computed bounds we get: 
\begin{equation}
  \label{eq:imcprz}
  Pr(\|S-\Wo\|_2\geq \so^1\gamma)\leq 2(d_1+d_2)\exp\left(-\frac{m\gamma^2}{\mu^4 d_1d_2(1+\gamma/3)}\right). 
\end{equation}
That is, w.p. $\geq 1-\gamma$: 
\begin{equation}
  \label{eq:imcprz1}
  \|S-\Wo\|_2 \leq \frac{\so^1\mu^2\sqrt{d_1d_2\log(2(d_1+d_2)/\gamma)}}{\sqrt{m}}. 
\end{equation}
Hence, by selecting $m=\Omega(\mu^4k^3\cdot \beta^2\cdot d_1d_2\log(2(d_1+d_2)/\gamma))$ where $\beta=\so^1/\so^k$, the following holds w.p. $\geq 1-\gamma$: 
$$\|S-\Wo\|_2 \leq \|\Wo\|_2\cdot \delta, $$
where $\delta= 1/(k^{3/2}\cdot \beta\cdot 100)$. 
\end{proof}
\begin{proof}[Proof of Property 2]
We prove the property for $B_y$; proof for $B_y$ follows analogously. Now, let $B_y=\frac{n_1n_2}{m}\sum_{(i,j)\in \Omega} Z_{ij}$ where $Z_i=\xit\u\u^T\x_i\yi\yti$. Then, 
\begin{align}
  \E[B_y]=\frac{n_1n_2}{m}\sum_{(i,j)\in \Omega} Z_{ij}=\frac{n_1n_2}{m}\sum_{i=1}^m\E_{(i,j)\in \Omega}[\xit\u\u^T\x_i\yi\yti]=I. 
\end{align}
Here again, we apply Theorem~\ref{thm:tropp} to bound $\|B_y-I\|_2$. To this end, we need to bound the following quantities: 
\begin{itemize}
\item {\bf Bound $\max_{ij} \|Z_{ij}\|_2$}: Now, $$\|Z_{ij}\|_2=\|\xit\u\u^T\x_i\yi\yti\|_2\leq \|\yi\|_2^2\|\x_i\|_2^2\leq \frac{\mu^4 d_1d_2}{n_1n_2}.$$
\item {\bf Bound $\|\sum_{(i,j)\in \Omega}E[Z_{ij}Z_{ij}^T]\|_2$ and  $\|\sum_{(i,j)\in \Omega}E[Z_{ij}^TZ_{ij}]\|_2$}: 

We first consider $\|\sum_{(i,j)\in \Omega}E[Z_{ij}Z_{ij}^T]\|_2$:
\begin{align}
\hspace*{-20pt}\left\|\sum_{(i,j)\in \Omega}\E[Z_{ij}Z_{ij}^T]\right\|_2&=\left\|\sum_{(i,j)\in \Omega}\E[(\xit\u\u^T\x_i)^2\|\y_i\|_2^2 \yi\yti]\right\|_2\stackrel{\zeta_1}{\leq}\frac{\mu^2d_2}{n_2}\left\|\sum_{(i,j)\in \Omega}\E[(\xit\u\u^T\x_i)^2\yi\yti]\right\|_2,\nonumber\\
&\stackrel{\zeta_2}{=}\frac{\mu^2d_2}{n_2^2}\left\|\sum_{(i,j)\in \Omega}\E[(\xit\u\u^T\x_i)^2\right\|_2\stackrel{\zeta_3}{\leq}\frac{\mu^4d_1d_2}{n_1n_2^2}\left\|\sum_{(i,j)\in \Omega}\E[(\xit\u)^2]\right\|_2\stackrel{\zeta_4}{=}\frac{\mu^4d_1d_2}{n_1^2n_2^2}\cdot m.
\end{align}
\end{itemize}
Note that the above given bounds that we obtain are exactly the same as the ones obtained in the Initialization Property's proof. Hence, by applying Theorem~\ref{thm:tropp} in a similar manner, and selecting $m=\Omega(k^3\beta^2 d_1\cdot d_2 \log(1/\gamma))$ and $\delta= 1/(k^{3/2}\cdot \beta\cdot 100)$, we get w.p. $\geq 1-\gamma$: 
$$\|B_y-I\|_2 \leq \delta.$$
Hence Proved. $\|B_x-I\|_2 \leq \delta$ can be proved similarly. 
\end{proof}
\begin{proof}[Proof of Property 3]
Note that $\E[C_y]=\E[\sum_{(i,j)\in \Omega}Z_{ij}]=0$.\\ Furthermore, both $\|Z_{ij}\|_2$ and $\|\E[\sum_{(i,j)\in \Omega} Z_{ij}Z_{ij}^T]\|_2$ have exactly the same bounds as those given in the Property 2's proof above. Hence, we obtain similar bounds. That is, if $m=\Omega(k^3\beta^2 d_1\cdot d_2 \log(1/\gamma))$ and $\delta= 1/(k^{3/2}\cdot \beta\cdot 100)$, we get w.p. $\geq 1-\gamma$: 
$$\|C_y\|_2 \leq \delta.$$
Hence Proved. $\|C_x\|_2$ can also be bounded analogously. 
\end{proof}

\end{proof}



%% file: multi.tex
\section{Multi-label Learning}\label{sec:mult}
In this section, we study the problem of multi-label regression with missing values. Let $X=[\x_1 \dots \x_{n_1}]\in \R^{d_1\times n_1}$ be the training matrix where  $\x_i$ is the feature vector of the $i$-th data point. Also, let $R\in \mathbb{R}^{n_1\times L}$ be the corresponding matrix of target variables. That is, $R^i=[R_{i1} \dots R_{ij} \dots R_{iL}]$ denotes $L$ target variables for $\x_i$. The goal is to learn a (low-rank) parameter matrix $\Wo$ s.t. $\tp{X}\Wo=R$.

The above problem is a straightforward multi-variate linear regression problem. However, in several large-scale multi-label learning problems, it is impossible to obtain all the target variables for each of the points. That is, $R$ generally has several entries missing. The goal is to learn $\Wo$ exactly, even when only a small number of random entries of $R$ is available. 

Here again, we view the problem as a low-rank matrix estimation problem with rank-one measurements $R_{ij}=\tp{\e}_i\tp{X}\Wo\e_j, (i,j)\in \Omega$, where index $\Omega$ is a uniformly sampled subset of $[n_1]\times [L]$. Note that this problem is a combination of the inductive matrix completion problem we studied in the previous section and the standard matrix completion. The left hand side measurement vector $X\e_i$ is similar to inductive matrix completion while the right hand measurement vector $\e_j$ is a standard matrix completion type of measurement vector. That is, this problem assumes the labels to be ``fixed'' but is inductive w.r.t. the data points $\x$. 

Similar to the previous section, we show that under a certain  incoherence assumption on the feature matrix $X$, Properties 1, 2, 3, required by Theorem~\ref{thm:general} are satisfied and hence alternating minimization will be able to learn the global optima $\Wo$. 
\begin{lemma}\label{lemma:mult}
Let  $X\in \mathbb{R}^{d_1\times n_1}$ be $\mu$-incoherent. Let $R=\tp{X}\Wo \in \mathbb{R}^{n_1\times L}$ be the ``labels'' matrix. 
Let $\Omega$ be a uniformly random subset of $[n_1]\times [L]$, s.t., $|\Omega|=m\geq C  \beta^2\cdot d_1 n_2\cdot \log(d_1+n_2)$, where $\beta=\sigma_R^1/\sigma_R^k$ is the condition number of $R$. Then, w.p. $\geq 1-1/(d_1+L)^{100}$, the measurement operators $A_{ij}=\sqrt{n_1n_2}\x_i \tp{\e}_j$ satisfy\footnote{We multiply $\x_i, \y_j$ by $\sqrt{n_1}, \sqrt{n_2}$ for normalization so that $\E_{i}[n_1\x_i\x_i^T]=I$ and $\E_{j}[n_2\yj\ytj]=I$} Properties 1,2,3 required by Theorem~\ref{thm:general}. 
\end{lemma}
Assuming $\beta, k$ to be constant and by ignoring log factors, the above lemma shows that using $m=d_1 \cdot L$ samples the parameter matrix $\Wo$ can be recovered exactly. In contrast, matrix completion requires  $m=n_1+L$ samples. That is, if the number of training points is significantly larger than $d_1\cdot L$, then the above method improves upon the matrix completion approach significantly. This result can be interpreted in another way: for  missing labels a standard method is to first do matrix completion and then learn $\Wo$. Our above lemma gives an example of a setting where  simultaneous learning and  completion of  $R$ leads to significantly better sample complexity. 

We now provide a proof of the above lemma. 
\begin{proof}
Here again, we divide the proof into three parts where each part proves a property mentioned in Theorem~\ref{thm:general}. 
\begin{proof}[Proof of Property 1]
As mentioned in the proof  of Lemma~\ref{lemma:indmc}, wlog, we can assume that  both $X, Y$ are orthonormal matrices and  that the condition number of $R$ is same as condition number of $\Wo$. 

We first recall the definition of $S$:
$$S=\frac{n_1n_2}{m}\sum_{(i,j)\in \Omega}^m \x_i\xti \Uo\So\Vot \e_j\etj=\frac{n_1n_2}{m}\sum_{(i,j)\in \Omega}^m Z_{ij},$$
where $Z_{ij}=\x_i\xti \Uo\So\Vot \e_j\etj = X \e_i \eit X^T \Uo\So\Vot \e_j \ejt $, where $\e_i, \e_j$ denotes the $i$-th, $j$-th canonical basis vectors, respectively.

Now using the fact that $(i,j)$ is sampled uniformly at random from $[n_1]\times [n_2]$: 
$$\E_{ij}[Z_{ij}]=\frac{1}{n_1n_2} XX^T\Uo\So\Vot=\Uo\So\Vot=\Wo/(n_1\cdot n_2),$$ where $XX^T=I$ follows by orthonormality of both $X$ and $Y$. 

As in the previous section, we first bound the following two quantities:
\begin{itemize}
\item {\bf Bound $\max_{ij} \|Z_{ij}\|_2$}: Now, $$\|Z_{ij}\|_2=\|\x_i\xti \Uo\So\Vot \e_j\etj\|_2\leq \so^1 \|\x_i\|_2^2 \leq \frac{\so^1\mu^2d_1}{n_1},$$
where the last inequality follows using incoherence of $X$ and $\Vo$.  
\item {\bf Bound $\|\sum_{(i,j)\in \Omega}E[Z_{ij}Z_{ij}^T]\|_2$}: 
\begin{align}
&\left\|\sum_{(i,j)\in \Omega}E[Z_{ij}Z_{ij}^T]\right\|_2=\left\|\sum_{(i,j)\in \Omega}\E[\x_i\xti\Wo\ej\etj\ej\etj\Wo^T \x_i \xti]\right\|_2,\nonumber\\
&\stackrel{\zeta_1}{=} \frac{1}{n_2} \left\|\sum_{(i,j)\in \Omega}\E[\x_i\xti\Wo\Wo^T\x_i \xti]\right\|_2\stackrel{\zeta_2}{\leq} \frac{(\so^1)^2\mu^2d_1}{n_1n_2}\left\|\sum_{(i,j)\in \Omega}\E[\x_i \xti]\right\|_2\stackrel{\zeta_3}{=} \frac{(\so^1)^2\mu^2d_1}{n_1^2n_2}\cdot m,\label{eq:mult_zub}
\end{align}
where $\zeta_1$ follows from $\E_j[\e_j\etj]=\frac{1}{n_2}I$, $\ \zeta_2$ follows from incoherence of $\x_i$, and $\zeta_3$ follows from $\E_i[\x_i\xti]=\frac{1}{n_1}I$.
\item {Bound $\|\sum_{(i,j)\in \Omega}E[Z_{ij}^TZ_{ij}]\|_2$}: 
\begin{align}
&\left\|\sum_{(i,j)\in \Omega}E[Z_{ij}^TZ_{ij}]\right\|_2=\left\|\sum_{(i,j)\in \Omega}\E[\ej\etj\Wot\x_i\xti\x_i\xti\Wo\ej\etj]\right\|_2,\nonumber\\
&\stackrel{\zeta_1}{\leq } \frac{\mu^2d_1}{n_1} \left\|\sum_{(i,j)\in \Omega}\E[\ej\etj\Wot\x_i\xti\Wo\ej\etj]\right\|_2\stackrel{\zeta_2}{=} \frac{\mu^2d_1}{n_1^2}\left\|\sum_{(i,j)\in \Omega}\E[\ej\etj\Wot\Wo\ej\etj]\right\|_2,\nonumber\\
&\stackrel{\zeta_3}{\leq} \frac{(\so^1)^2\mu^2d_1}{n_1^2n_2}\cdot m,\label{eq:mult_zub1}
\end{align}
where $\zeta_1$ follows from incoherence of $X$, $\zeta_2, \zeta_3$ follows from uniform sampling of $\e_i$ and $\e_j$, respectively. 
\end{itemize}
Using  \eqref{eq:mult_zub}, \eqref{eq:mult_zub1} we get: $$\max\left(\left\|\sum_{(i,j)\in \Omega}E[Z_{ij}Z_{ij}^T]\right\|_2,\ \left\|\sum_{(i,j)\in \Omega}E[Z_{ij}^TZ_{ij}]\right\|_2\right)\leq \frac{(\so^1)^2\mu^2d_1}{n_1^2n_2}\cdot m.$$

\noindent Using the above given bounds, and Theorem~\ref{thm:tropp}, we get: 
\begin{equation}
  \label{eq:multprz}
  Pr(\|S-\Wo\|_2\geq \frac{n_2\so^1\gamma}{\mu\sqrt{k}})\leq 2(d_1+n_2)\exp\left(-\frac{m\gamma^2}{\mu^4\cdot k\cdot d_1(1+\gamma/3)}\right). 
\end{equation}
That is, by selecting $m=\Omega(k^3 \beta^2 \mu^2 d_1n_2\log(2(d_1+n_2)/\gamma)$ with $\beta=\frac{\so^1}{\so^k}$, the following holds w.p. $\geq 1-\gamma$: 
$$\|S-\Wo\|\leq \delta\|\Wo\|_2,$$ 
where $\delta\leq \frac{1}{k^{3/2}\cdot \beta \cdot 100} $. 
\end{proof}
\begin{proof}[Proof of Property 2]
Here, we first prove the property for $B_y$. Now, $B_y=\frac{n_1n_2}{m}\sum_{(i,j)\in \Omega} Z_{ij}$ where $Z_i=\xit\u\u^T\x_i\ej\etj$. Note that, $\E[B_y]=I$. 

Next, we bound the quantities required by Theorem~\ref{thm:tropp}: 
\begin{itemize}
\item {\bf Bound $\max_{ij} \|Z_{ij}\|_2$}: Now, $$\|Z_{ij}\|_2=\|\xit\u\u^T\x_i \e_j\etj\|_2\leq \|\x_i\|_2^2 \leq \frac{\mu^2d_1}{n_1},$$
where the second inequality follows from incoherence of $X$. 
\item {\bf Bound $\|\sum_{(i,j)\in \Omega}E[Z_{ij}Z_{ij}^T]\|_2$}: 
\begin{align*}
\hspace*{-20pt}\left\|\sum_{(i,j)\in \Omega}E[Z_{ij}Z_{ij}^T]\right\|_2=\left\|\sum_{(i,j)\in \Omega}\E[(\xit\u\u^T\x_i)^2 \e_j\etj]\right\|_2\stackrel{\zeta_1}{=}\frac{1}{n_2}\sum_{(i,j)\in \Omega}\E[(\xit\u\u^T\x_i)^2]\stackrel{\zeta_2}{\leq} \frac{\mu^2 d_1}{n_1^2n_2},
\end{align*}
where $\zeta_1$ follows as $\e_j$ is sampled uniformly and $\zeta_2$ follows by using incoherence of $X$ and uniform sampling of $\e_i$. 

Hence, using $m=\Omega(k^3 \cdot \beta^2 \cdot d \cdot n_2 \log(2(d_1+n_2)/\gamma)$, then we have (w.p. $\geq 1-\gamma$):
$$\|B_y-I\|_2\leq \delta,$$
where $\delta=1/(k^{3/2}\cdot \beta\cdot 100)$.
\end{itemize}
Now, we bound $B_x=\frac{n_1n_2}{m}\sum_{(i,j)\in \Omega} Z_{ij}$ where $Z_i=\etj\v\v^T\ej\x_i\xit$. Note that, $\E[B_y]=I$. 
Next, we bound the quantities required by Theorem~\ref{thm:tropp}: 
\begin{itemize}
\item {\bf Bound $\max_{ij} \|Z_{ij}\|_2$}: Now, $$\|Z_{ij}\|_2=\|\etj\v\v^T\ej\x_i\xit\|_2\leq \|\x_i\|_2^2 \leq \frac{\mu^2d_1}{n_1},$$
where the second inequality follows from incoherence of $X$. 
\item {\bf Bound $\|\sum_{(i,j)\in \Omega}E[Z_{ij}Z_{ij}^T]\|_2$}: 
\begin{align}
\left\|\sum_{(i,j)\in \Omega}E[Z_{ij}^TZ_{ij}]\right\|_2&=\left\|\sum_{(i,j)\in \Omega}\E[(\etj\v\v^T\ej)^2\|\x_i\|^2\x_i\xit]\right\|_2,\nonumber\\
&\stackrel{\zeta_1}{\leq }\frac{1}{n_2}\sum_{(i,j)\in \Omega}\E[\|\x_i\|^2\x_i\xit],\nonumber\\
&\stackrel{\zeta_2}{\leq} \frac{\mu^2 d_1}{n_1^2n_2},
\end{align}
where $\zeta_1$ follows as $\e_j$ is sampled uniformly and $\zeta_2$ follows by using incoherence of $X$ and uniform sampling of $\e_i$. 

Hence, using $m=\Omega(k^3 \cdot \beta^2 \cdot d \cdot n_2 \log(2(d_1+n_2)/\gamma)$,  we have (w.p. $\geq 1-\gamma$):
$$\|B_x-I\|_2\leq \delta,$$
where $\delta=1/(k^{3/2}\cdot \beta\cdot 100)$.
\end{itemize}
\end{proof}
\begin{proof}[Proof of Property 3]
We first note that $\E[C_x]=\E[C_y]=0$. Now, here again we use Theorem~\ref{thm:tropp} to say that $C_x, C_y$  converge to their mean. The quantities we need to bound are similar to the ones proved above for Property 2. Hence, the Property 3 follows using  $m=\Omega(k^3 \cdot \beta^2 \cdot d \cdot n_2 \log(2(d_1+n_2)/\gamma)$ samples. 
\end{proof}
\end{proof}

%% file: exps.tex
\section{Experiments}\label{sec:exps}
\begin{figure}[t]
  \centering
  \begin{tabular}[t]{cccc}
    \includegraphics[width=.25\textwidth]{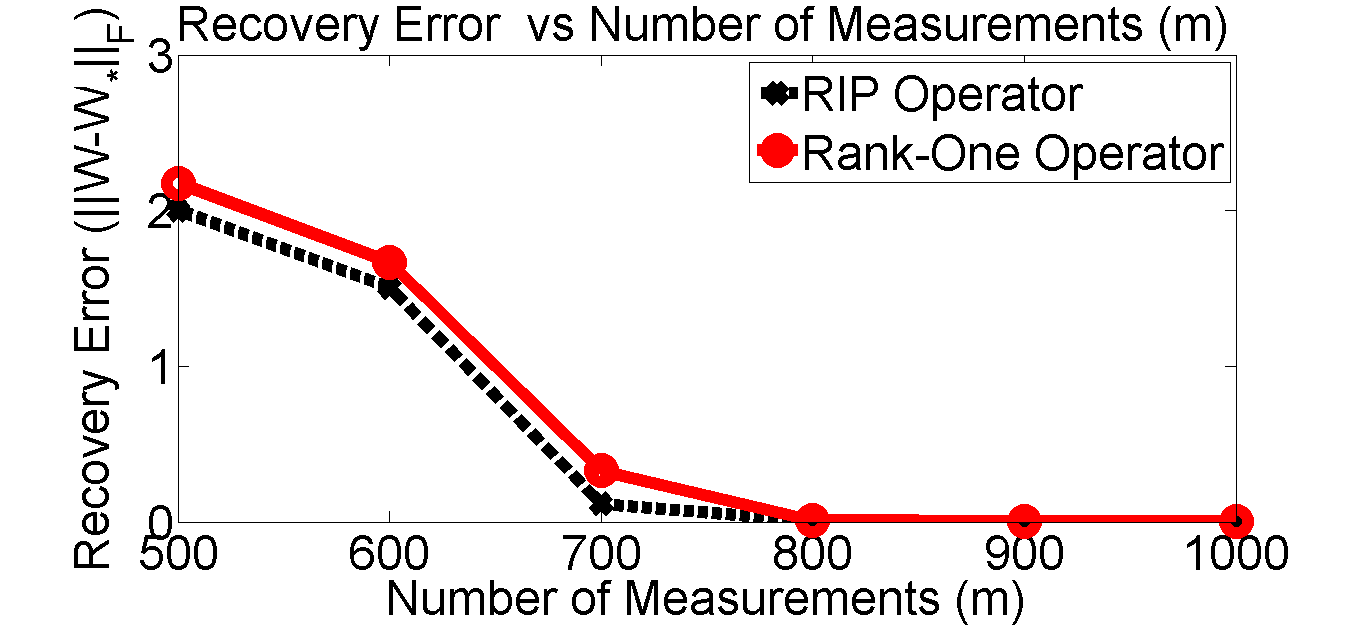}&\hspace*{-15pt}
    \includegraphics[width=.25\textwidth]{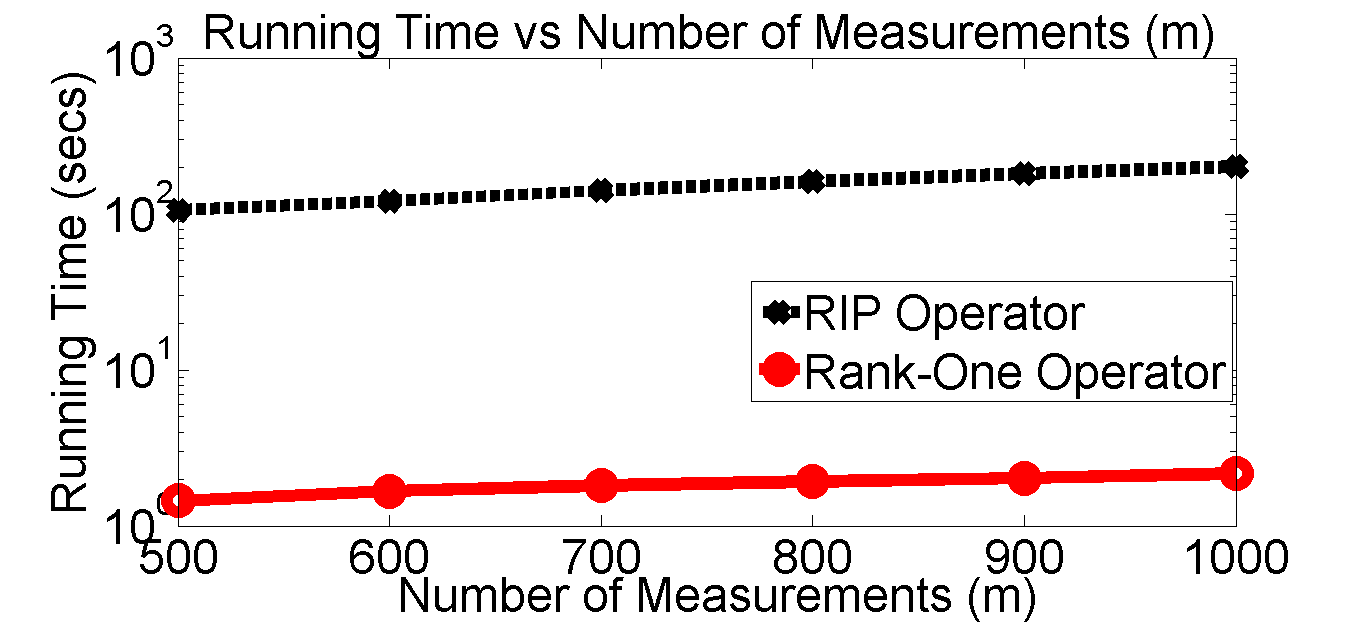}&\hspace*{-15pt}
    \includegraphics[width=.25\textwidth]{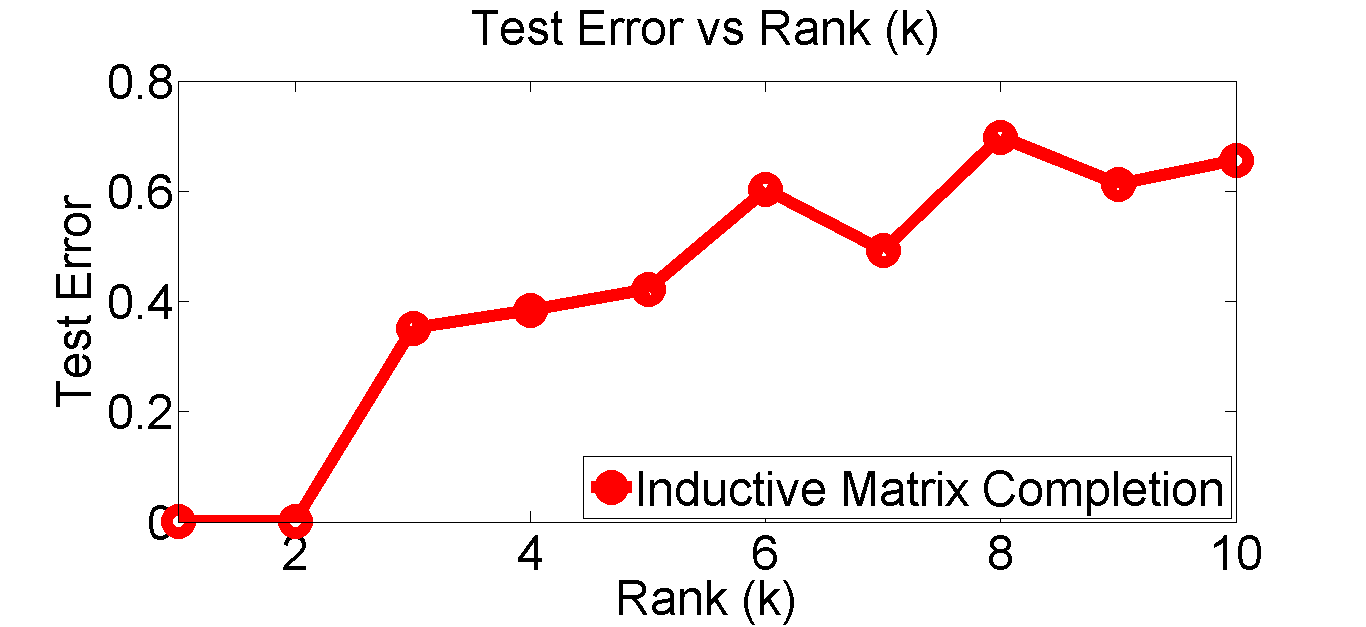}&\hspace*{-15pt}
    \includegraphics[width=.25\textwidth]{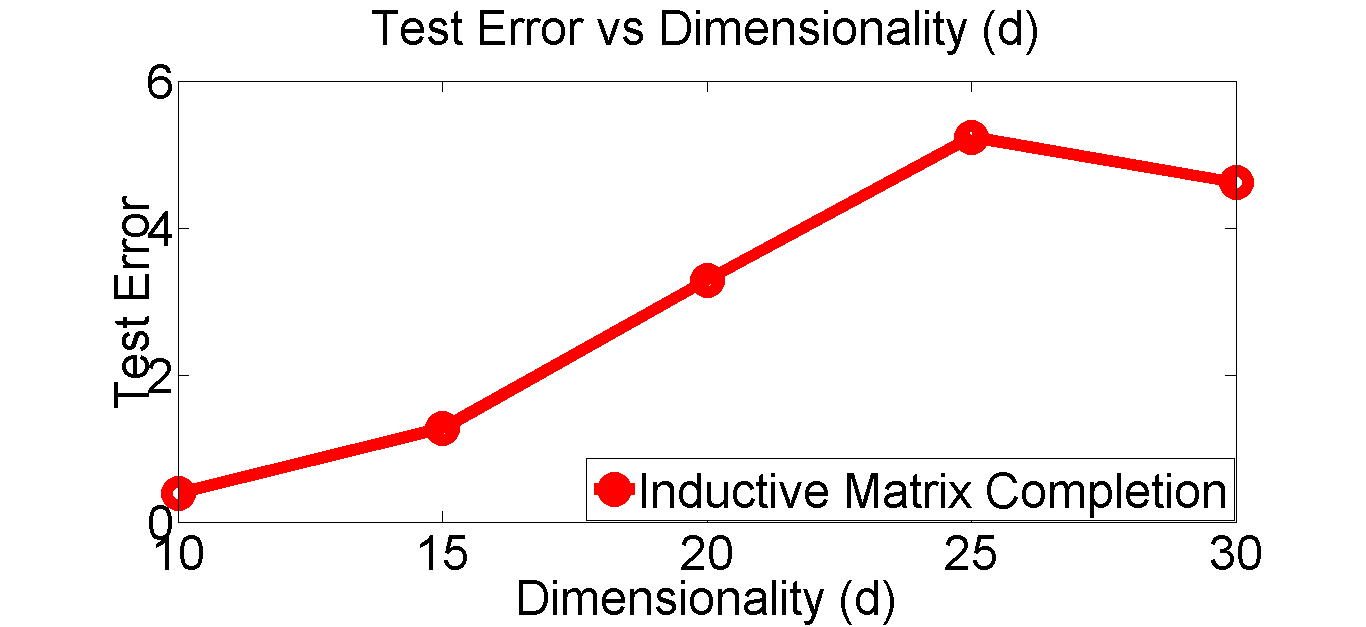}\\
{\bf (a)}&{\bf (b)}&{\bf (c)}&{\bf (d)}
  \end{tabular}
\caption{{\bf (a), (b)}: Low-rank Matrix Sensing---Comparison of RIP based and the rank-one matrices based measurement operators for  low-rank matrix sensing. Clearly, our rank-one operator is significantly faster than the RIP based method while incurring similar recovery error. {\bf (c), (d)}: Inductive Matrix Completion---plots show the error incurred by  alternating minimization  on the test data with, (c): varying rank of the underlying $\Wo$,  and (d): varying dimensionality  of $\Wo$.}\label{fig:exp}
\end{figure}
In this section, we first demonstrate empirically that our Gaussian rank-one linear operator ($\calA_{Gauss}$) is significantly more efficient for matrix sensing than the existing RIP based measurement operators. To this end, we first generated a random rank-$5$ signal $\Wo \in \mathbb{R}^{50\times 50}$ and then generate different number of measurements using both $\calA_{Gauss}$  and an RIP based operator. We run alternating minimization method for both type of measurements. Figure~\ref{fig:exp} (a) compares the Frobenius norm in recovery by both the methods. Figure~\ref{fig:exp} (b) plots (on log-scale) the running time of both the methods as $m$ increases. Clearly, the $\calA_{Gauss}$ operator based measurements provide reasonably accurate recovery while the running time of our $\calA_{Gauss}$ based method is about two orders of magnitude better than that of RIP based measurement method.

Next, we demonstrate that by using a very small number of measurements, the multi-label regression problem can still be solved accurately. For this, we selected number of labels $L=50$, number of points $n_1=100$, and varied $d$ from $1$ to $20$. We then generated $100$ training points $X\in \mathbb{R}^{d_1\times 100}$ and $100$ test points. We then generated $\Wo\in \mathbb{R}^{d_1\times L}$ and observed only $200$ random entries of $R=X^T\Wo$.  Figure~\ref{fig:exp} (c), (d) plot the error incurred in prediction over the test set, as $k$ and $d$ vary respectively. The error is computed using $\sum_{\x\in Test Set}|R_{xj}-\x^T\Wo\e_j|^2$. Clearly, the method is able to output fairly accurate predictions for small $k, d$. Moreover, the test error degrades gracefully as either $k$ or $d$ increases. 

%% file: app_sense.tex
\section{Preliminaries}\label{app:prelim}
\begin{theorem}[Theorem 1.6 of \cite{Tropp11}]\label{thm:tropp}
Consider a finite sequence $Z_i$ of independent, random matrices with dimensions $d_1\times d_2$. Assume that each random matrix satisfies $\mathbb{E}[Z_i] = 0$ and $\|Z_i\|_2\leq R$ almost surely. Define, $\sigma^2 := \max \{\|\sum_i\mathbb{E}[Z_iZ_i^T]\|_2,\|\sum_i\mathbb{E}[Z_i^TZ_i]\|_2\}.$ Then, for all  $\gamma\geq 0$, $$\mathbb{P}\left(\left\|\frac{1}{m}\sum_{i=1}^mZ_i\right\|_2\geq \gamma\right)\leq (d_1+d_2)\exp\left(\frac{-m^2\gamma^2}{\sigma^2+Rm\gamma/3}\right).$$
\end{theorem}
\section{Proof of General Theorem for Low-rank Matrix Estimation}\label{app:general}
Here, we now generalize our above given proof to the rank-$k$ case. In the case of rank-$1$ matrix recovery, we used $1-(\vtnt\uo)^2$ as the error or distance function and show at each step that the error decreases by at least a constant factor. For general rank-$k$ case, we need to generalize the distance function to be a distance over subspaces of dimension-$k$. To this end, we use the standard principle angle based subspace distance. That is,
\begin{defn}
  Let $U_1, U_2 \in \mathbb{R}^{d\times k}$ be $k$-dimensional subspaces. Then the principle angle based distance $\dist(U_1, U_2)$ between $U_1, U_2$ is given by: $$\dist(U_1, U_2)=\|U_\perp^TU_2\|_2,$$ where $U_\perp$ is the subspace orthogonal to $U_1$. 
\end{defn}
\begin{proof}[Proof of Theorem~\ref{thm:general}: General Rank-$k$ Case]
For simplicity of notation, we denote $U_h$ by $U$, $\widehat{V}_{h+1}$ by $\widehat{V}$, and $V_{h+1}$ by $V$. 

Similar to the above given proof, we first present the update equation for $\widehat{V}_{(t+1)}$. 
Recall that $\widehat{V}_{(t+1)}=\argmin_{V\in \R^{d_2\times k}}  \sum_i (\xit\Wo\yi-\xit\Ut \tp{\Vh}\yi)^2$. Hence, by setting gradient of this objective function to $0$, using the above given notation and by simplifications, we get: 
\begin{equation}\label{eq:vh}\Vh=\tp{\Wo}U-F,\end{equation}
where $F=[F_1 F_2 \dots F_k]$ is the ``error'' matrix. 

Before specifying $F$, we first introduce {\em block matrices} $B, C, D, S \in \mathbb{R}^{k d_2 \times k d_2}$ with $(p,q)$-th block $B_{pq}, C_{pq}, S_{pq}, D_{pq}$ given by: 
\begin{align}  \label{eq:Bpq}
  B_{pq}&=\sum_i \yi\yit (\xit \u_{p})(\xit \u_{q}),\\
  C_{pq}&=\sum_i \yi\yit (\xit \u_{p})(\xit \u_{*q}),\label{eq:Cpq}\\
D_{pq}&=\u_p^T\u_{*q}I,\label{eq:Dpq}\\
S_{pq}&=\so^p I \ \  \text{ if }p=q,\ \ \ \ \text{and }\ \ 0\ \  \text{ if }p\neq q.\label{eq:Spq}
\end{align}
where $\so^p=\So(p,p)$, i.e., the $p$-th singular value of $\Wo$ and $\u_{*q}$ is the $q$-th column of $\Uo$. 

Then, using the definitions given above, we get: 
\begin{equation}
  \label{eq:Fi}
  \left[\begin{matrix}F_1\\\vdots\\ F_k\end{matrix}\right]=B^{-1}(BD-C)S\cdot \mvec(\Vo).
\end{equation}

Now, recall that in the $t+1$-th iteration of Algorithm~\ref{alg:altmin}, $V_{t+1}$ is obtained by QR decomposition of $\widehat{V}_{t+1}$. Using notation mentioned above, $\Vh=V R$ where $R$ denotes the lower triangular matrix $R_{t+1}$ obtained by the QR decomposition of $V_{t+1}$. 

Now, using \eqref{eq:vh}, $V=\Vh R^{-1}=(\Wo^TU-F)R^{-1}.$ Multiplying both the sides by $\Vop$, where $\Vop$ is a fixed orthonormal basis of the subspace orthogonal to $span(\Vo)$, we get: 
\begin{equation}
  \label{eq:t2}
  (\Vop)^TV=-(\Vop)^TFR^{-1} \Rightarrow dist(\Vo, V_{t+1})=\|(\Vop)^TV\|_2 \leq \|F\|_2\|R^{-1}\|_2. 
\end{equation}
Also, note that using the initialization property (1) mentioned in Theorem~\ref{thm:general}, we get $\|S-\Wo\|_2\leq \frac{\so^k}{100}$. Now, using the standard sin theta theorem for singular vector perturbation\cite{Li94}, we get: $\dist(U_0, \Uo)\leq \frac{1}{100}$. 

Theorem now follows by using Lemma~\ref{lem:F}, Lemma~\ref{lem:R} along with the above mentioned bound on $\dist(U_0, \Uo)$. 



\end{proof}
\begin{lemma}
  \label{lem:F}
  Let $\calA$ be a rank-one measurement operator where $A_i=\x_i\tp{\y}_i$. Also, let $\calA$ satisfy Property 1, 2, 3 mentioned in Theorem~\ref{thm:general} and let $\so^1\geq \so^2\geq \dots \geq\so^k$ be the singular values of $\Wo$. Then, $$\|F\|_2\leq \frac{\so^k}{100}\dist(U_t, \Uo). $$
\end{lemma}

\begin{lemma}
  \label{lem:R}
  Let $\calA$ be a rank-one measurement operator where $A_i=\x_i\tp{\y}_i$. Also, let $\calA$ satisfy Property 1, 2, 3 mentioned in Theorem~\ref{thm:general}. Then, $$\|R^{-1}\|_2\leq\frac{1}{ \so^k\cdot \sqrt{1-dist^2(U_t, \Uo)}-\|F\|_2}.$$ 
\end{lemma}

\begin{proof}[Proof of Lemma~\ref{lem:F}]
Recall that $\mvec(F)=B^{-1}(BD-C)S\cdot \mvec(\Vo)$. Hence, 
\begin{equation}\label{eq:f1}\|F\|_2\leq \|F\|_F\leq \|B^{-1}\|_2\|BD-C\|_2\|S\|_2\|\mvec(\Vo)\|_2=\so^1\sqrt{k}\|B^{-1}\|_2\|BD-C\|_2.\end{equation}
Now, we first bound $\|B^{-1}\|_2=1/(\sigma_{\text{min}}(B))$. 
Also, let $Z=[\z_1 \z_2 \dots \z_k]$ and let $\z=\mvec(Z)$. Then, 
\begin{align}
  \sigma_{\text{min}}(B)&=\min_{\z, \|\z\|_2=1}\z^TB\z=\min_{\z, \|\z\|_2=1}\sum_{1\leq p\leq k, 1\leq q\leq k}\z_p^TB_{pq}\z_q\nonumber\\&=\min_{\z, \|\z\|_2=1}\sum_{p}\z_p^TB_{pp}\z_p+\sum_{pq, p\neq q}\z_p^TB_{pq}\z_q. \label{eq:bp0}
\end{align}
Recall that, $B_{pp}=\frac{1}{m}\sum_{i=1}^m \yi\yit (\xit\u_p)^2$ and $\u_p$ is independent of $\xi, \yi, \forall i$. Hence, using Property 2 given in Theorem~\ref{thm:general}, we get: 
\begin{equation}
  \label{eq:bp1}
  \sigma_{\text{min}}(B_{pp})\geq 1- \delta,
\end{equation}
where, $$\delta=\frac{1}{k^{3/2}\cdot \beta\cdot 100},$$ and $\beta=\so^1/\so^k$ is the condition number of $\Wo$. 

Similarly, using Property (3), we get: 
\begin{equation}
  \label{eq:bp2}
  \|B_{pq}\|_2\leq \delta. 
\end{equation}
Hence, using \eqref{eq:bp0}, \eqref{eq:bp1}, \eqref{eq:bp2}, we get:
\begin{equation}
  \label{eq:bp3}
  \sigma_{\text{min}}(B)\geq \min_{\z, \|\z\|_2=1}(1-\delta)\sum_{p}\|\z_p\|_2^2-\delta\sum_{pq, p\neq q}\|\z_p\|_2\|\z_q\|_2=\min_{\z, \|\z\|_2=1}1-\delta\sum_{pq}\|\z_p\|_2\|\z_q\|_2\geq 1-k\delta. 
\end{equation}

Now, consider $BD-C$: 
\begin{align}
  \|BD-C\|_2&= \max_{\z, \|\z\|_2=1}|\z^T(BD-C)\z|,\nonumber\\
&=\max_{\z, \|\z\|_2=1}\left|\sum_{1\leq p\leq k, 1\leq q\leq k}\z_p^T\yi\yit\z_q\xit\left(\sum_{1\leq \ell\leq k}\ip{\u_\ell}{\u_{*q}}\u_p\u_\ell^T-\u_{p}\u_{*q}^T\right)\x_i\right|,\nonumber\\
&=\max_{\z, \|\z\|_2=1}\left|\sum_{1\leq p\leq k, 1\leq q\leq k}\z_p^T\yi\yit\z_q\xit\u_p\u_{*q}^T(UU^T-I)\x_i\right|,\nonumber\\
&\stackrel{\zeta_1}{\leq}\delta \max_{\z, \|\z\|_2=1} \sum_{1\leq p\leq k, 1\leq q\leq k} \|(UU^T-I)\u_{*q}\|_2\|\z_p\|_2\|\z_q\|_2\leq k\cdot \delta\cdot \dist(U, \Uo),\label{eq:bp4}
\end{align}
where $\zeta_1$ follows by observing that $\u_{*q}^T(UU^T-I) \u_p=0$ and then by applying Property (3) mentioned in Theorem~\ref{thm:general}. 

Lemma now follows by using \eqref{eq:bp4} along with  \eqref{eq:f1} and \eqref{eq:bp3}. 
\end{proof}
\begin{proof}[Proof of Lemma~\ref{lem:R}]
The lemma is exactly the same as Lemma 4.7 of \cite{JainNS13}. We reproduce their proof here for completeness. 

Let $\sigma_{\text{min}}(R)$ be the smallest singular value of $R$, then: 
\begin{align}
  \sigma_{\text{min}}(R)&=\min_{\z, \|\z\|_2=1}\|R\z\|_2=\min_{z, \|z\|_2=1}\|VR\z\|_2=\min_{\z, \|\z\|_2=1}\|\Vo\So\Uot  U\z-F\z\|_2, \nonumber\\
&\geq \min_{\z, \|\z\|_2=1}\|\Vo\So\Uot  U\z\|_2-\|F\z\|_2\geq \so^k\sigma_{\text{min}}(U^T\Uo)-\|F\|_2,\nonumber\\
&\geq \so^k\sqrt{1-\twonorm{U^T\Uo^\perp}^2}-\|F\|_2=\so^k\sqrt{1-\dist(\Uo, U)^2}-\|F\|_2. 
\end{align}
Lemma now follows by using the above inequality along with the fact that $\|R^{-1}\|_2\leq 1/\sigma_{\text{min}}(R)$.
\end{proof}